\documentclass[12pt]{article} 

\RequirePackage[OT1]{fontenc}
\RequirePackage{amsthm,amsmath,natbib}

\usepackage{graphicx,enumitem}

\usepackage{subcaption}
\usepackage[colorlinks = true,
            linkcolor = blue,
            urlcolor  = blue,
            citecolor = brown,
            anchorcolor = blue]{hyperref}
            \usepackage{tcolorbox}
            
\usepackage[utf8]{inputenc}
\usepackage{amsmath}
\usepackage[english]{babel}
\usepackage{amssymb}
\usepackage{amsfonts}
\usepackage{bm}
\usepackage{graphicx}
\usepackage{amsmath}
\usepackage{amsthm}
\usepackage{xcolor}
\usepackage{parskip}
\usepackage{multicol}
\usepackage{float}

\floatstyle{plaintop}
\restylefloat{table}
\usepackage{setspace}
\usepackage{rotating}

\newtheorem{remark}{Remark}[section]
\newtheorem{mydef}{Definition}[section]
\newtheorem{teo}{Theorem}[section]

\newtheorem{lemma}[teo]{Lemma}
\newtheorem{prop}[teo]{Proposition}

\newmuskip\pFqmuskip
\newcommand*\pFq[6][8]{%
  \begingroup 
  \pFqmuskip=#1mu\relax
  \mathcode`\,=\string"8000
  \begingroup\lccode`\~=`\,
  \lowercase{\endgroup\let~}\pFqcomma
  {}_{#2}F_{#3}{\left[\genfrac..{0pt}{}{#4}{#5};#6\right]}%
  \endgroup
}
\newcommand{\pFqcomma}{\mskip\pFqmuskip}

\def\S{\mathbb{S}}
\def\R{\mathbb{R}}
\def\Z{\mathbb{Z}}

\def\spacingset#1{\renewcommand{\baselinestretch}%
{#1}\small\normalsize} \spacingset{1}

\begin{document}

\singlespace

\baselineskip=28pt \vskip 2cm
\begin{center} {\LARGE{\sc Sobolev Spaces, Kernels and Discrepancies over Hyperspheres}}
\end{center}

\vspace{2cm}
\begin{center}\large
Simon Hubbert,\footnote{ \baselineskip=10pt
Department of Economics, Mathematics and Statistics, Birkbeck, University of London UK 
}
Emilio Porcu,\footnote{ \baselineskip=10pt
Department of Mathematics at Khalifa University, Abu Dhabi 
}
Chris. J. Oates \footnote{ \baselineskip=10pt
School of Mathematics, Statistics \& Physics, Newcastle University, UK;
Alan Turing Institute, UK}
and Mark Girolami \footnote{ \baselineskip=10pt
Department of Engineering, University of Cambridge, UK;
Alan Turing Institute, UK}
\end{center}

\vspace{2cm}

\begin{abstract}
This work provides theoretical foundations for kernel methods in the hyperspherical context. 
Specifically, we characterise the native spaces (reproducing kernel Hilbert spaces) and the Sobolev spaces associated with kernels defined over hyperspheres. 
Our results have direct consequences for kernel cubature, determining the rate of convergence of the worst case error, and expanding the applicability of cubature algorithms based on Stein's method. 
We first introduce a suitable characterisation on Sobolev spaces on the $d$-dimensional hypersphere embedded in $(d+1)$-dimensional Euclidean spaces. Our characterisation is based on the Fourier--Schoenberg sequences associated with a given kernel. Such sequences are hard (if not impossible) to compute analytically on $d$-dimensional spheres, but often feasible over Hilbert spheres. We circumvent this problem by finding a projection operator that allows to Fourier mapping from Hilbert into finite dimensional hyperspheres. We illustrate our findings through some parametric families of kernels.  \\

\vspace{0.3cm}
\noindent 

\end{abstract}





\newpage
\singlespace
\section{Introduction}
\label{sec:intro}



The aim of this paper is to precisely characterise the Hilbert spaces reproduced by certain kernels defined over the hypersphere $\mathbb{S}^d$, the $d$-dimensional sphere embedded in $\mathbb{R}^{d+1}$. 
A kernel $K : \mathbb{S}^d \times \mathbb{S}^d \rightarrow \mathbb{R}$ is said to be \emph{geodesically isotropic} if it depends only on the geodesic distance (denoted $\theta$ throughout) between any pair of points located over the spherical shell. 
Hence, under geodesic isotropy, there exists a mapping $\psi: [0,\pi] \to \mathbb{R}$ such that, for all $\boldsymbol{\xi},\boldsymbol{\eta} \in \mathbb{S}^d$, $K(\boldsymbol{\xi},\boldsymbol{\eta})= \psi(\theta(\boldsymbol{\xi},\boldsymbol{\eta}))$.
Is there a suitable spectral characterisation of Sobolev-type kernels, $K$, under the assumption of geodesic isotropy? 
The literature is elusive on this aspect. 
Several contributions refer to isotropic kernels in $\mathbb{R}^{d+1}$ and claim that the restriction of those kernels to $\mathbb{S}^d$ retains the same Sobolev properties. 
The statement is correct, albeit quite unnatural: constructive criticism in \cite{gneiting2013} and \cite{PAF2016} show that direct constructions on the sphere would be preferable. 
To do so, it becomes necessary to properly define Sobolev spaces for geodesically isotropic kernels over spheres.

\subsection{Motivation:  Kernels, Cubature and Discrepancies}

The applications of reproducing kernels are myriad, but the principal motivation for this work is to provide theoretical foundations for the related computational tools of \emph{kernel cubature} and \emph{kernel discrepancy}. 
We illustrate these concepts for the general case of kernels $K : \mathcal{X} \times \mathcal{X} \rightarrow \mathbb{R}$ on a measurable space $\mathcal{X}$, with no assumptions on their geometry. 
We denote $\mathcal{N}_{K}$ the reproducing kernel Hilbert space associated with $K$.
Let $\delta(x)$ denote a unit mass at $x \in \mathcal{X}$.
Given a probability distribution $P$ on $\mathcal{X}$ and a set of locations $\{x_1,\dots,x_n\} \subset \mathcal{X}$, kernel cubature associates to each location $x_i$ a scalar weight $w_i$, such that the kernel discrepancy
\begin{align}
D_k\left( \sum_{i=1}^n w_i \delta(x_i) , P \right) = \sup_{\|f\|_{\mathcal{N}_K} \leq 1} \left| \sum_{i=1}^n w_i f(x_i) - \int f \; \mathrm{d}P \right| \label{eq: kernel discrepancy}
\end{align}
is minimised.
These topics have received considerable recent interest in statistics, machine learning, and numerical analysis, where kernel cubature has been applied to such tasks as sampling \citep{teymur2021optimal}, experimental design \citep{pronzato2020bayesian}, model selection \citep{briol2019probabilistic}, and numerical integration \citep{jagadeeswaran2019fast}.
The increasing popularity of kernel cubature is due in part to a closed-form expressions for the cubature weights $w_i$ and to the fact that their error decay is rate-optimal among all cubature methods for $\mathcal{N}_{K}$, even when the locations $x_i$ are randomly sampled.
Further, these tools have received attention in connection with Stein's method from applied probability \citep{Stein1972}, where \citet{chwialkowski2016kernel,liu2016kernelized} introduced a kernel discrepancy that can be computed even when $P$ is implicitly defined up to a normalisation constant, addressing a problem that is routinely encountered in the Bayesian context.

Specialising the discussion to $\mathcal{X} = \mathbb{S}^d$, kernel cubature appears in rendering algorithms for glossy surfaces \citep{marques2013spherical,marques2022gaussian} and as a criterion by which the performance of rendering algorithms is measured \citep{marques2019spectral}, while kernel cubature has been used in combination with Stein's method to numerically approximate posterior expectations in directional statistics \citep{barp2022riemann}.
The theory of kernel cubature on is well-developed \citep[e.g. as a special case of the general theory of][]{novak2008tractability}, and optimality properties of kernel cubature have been established in the case where the Hilbert space reproduced by the kernel is equivalent to a Sobolev space on $\mathbb{S}^d$ \citep{krieg2021function}.
However, applications of kernel cubature on $\mathbb{S}^d$ are limited by the availability of kernels that satisfy theoretical assumptions and are computationally practical.
For example, the methodology of \citet{barp2022riemann} requires $K$ to reproduce a Hilbert space equivalent to an order-$\beta$ Sobolev space and to admit computable expressions for its derivatives in order for Stein's method to be applied.
Further, it is desirable from the perspective of empirical performance for the kernel to be \emph{intrinsically} defined on $\mathbb{S}^d$, to properly reflect the geometry of $\mathbb{S}^d$.
If such kernels can be found, then the \emph{Riemann--Stein kernel method} of \citet{barp2022riemann} facilitates approximation of posterior expectations with error $O(n^{-\beta/d})$, improving on the conventional $O(n^{-1})$ error of Markov chain Monte Carlo whenever $\beta > d/2$ (i.e. when the Sobolev--H\"{o}lder embedding is well-defined).
An improved understanding of Sobolev-type kernels therefore has the potential to eliminate the gap between the theory and practice of kernel cubature and kernel discrepancy on $\mathbb{S}^d$, and to drive the improvement of methodologies in application areas such as graphics rendering and directional statistics where kernel cubature and kernel discrepancies are used.

\subsection{Contribution}

The present article contributes to mathematical understanding of Sobolev-type kernels on $\mathbb{S}^d$, making in particular the following relevant contributions: 
\begin{description}
\item[Native and Sobolev Spaces.] We start with the native spaces (reproducing kernel Hilbert spaces) associated with kernels that are geodesically isotropic over $d$-dimensional spheres. We then define suitable Sobolev spaces with exponent $\beta$ for the case of $d$-dimensional spheres. 
Using Fourier analysis over spheres, we prove that a given kernel, $K$, belongs to a given Sobolev space with exponent $\beta$ if and only if the related Fourier--Schoenberg sequences (see subsequent sections) have a precise rate of decay. 
\item[Hilbert Spheres and Projections.] For $d$ and $K$ given, attaining the Fourier--Schoenberg sequence is extremely difficult. After noting that such sequences are more easily attainable on the (infinite dimensional) Hilbert sphere, we prove that there exists a projection operator that relates those sequences on the Hilbert sphere with their analogue on finite dimensional sphere. 

\item[Euclidean Kernels and Their Restriction.] One potential source of geodesically isotropic kernels, as alluded to above, is to simply restrict a radially isotropic kernel in the ambient Euclidean space $\R^{d+1}$ to the $d-$dimensional sphere. If one knows in advance the Fourier transform of the radial kernel then one can employ a
 formula due to  \cite{narcowich2002scattered}  to derive the corresponding Fourier--Schoenberg sequence of its restriction to the sphere.

\item[Some Parametric Classes of Kernels.] The above ingredients will allow us to prove that celebrated classes of kernels for $d$-dimensional spheres are actually Sobolev kernels, and can be used within the kernel cubature and kernel discrepancy machinery as previously described. 
\end{description} 
The plan of the paper is the following. Section \ref{sec2} provides a succinct mathematical background. Section \ref{sec3} starts with expository material on native spaces and then propose a definition of Sobolev space with given exponent. We prove that the Fourier--Schoenberg sequences determine the Sobolev space where the kernel {\em sits}. Section \ref{sec4} sets out a framework which delivers closed form expressions for Fourier--Schoenberg sequences derived via the  two routes described above, i.e., by projection from the Hilbert sphere and also by restricting a Euclidean radial function to the sphere. We conclude in Section \ref{sec5} by presenting three explicit parametric cases. Using the tools developed in Section \ref{sec4} we give closed form expressions for the Fourier--Schoenberg coefficients, we also provide their asymptotic rates of decay and hence conclude by specifying their corresponding Sobolev spaces. A short discussion in Section \ref{secdisc} closes the paper. The proofs of these results are technical and involved, so  they are deferred to a long Appendix.

\color{black}

\section{Background} \label{sec2}

Let  $d$ be a positive integer. We consider the $d$-dimensional unit sphere   $\mathbb{S}^d=\{\bm{\xi}\in\mathbb{R}^{d+1}, \|\bm{\xi}\|=1\}$,  embedded in $\R^{d+1}$, with $\|\cdot\|$ denoting Euclidean norm.  We shall also refer to the Hilbert sphere $ \mathbb{S}^{\infty} = \{\bm{\xi} \in\mathbb{R}^{\mathbb{N}}: \|\bm{x}\|=1\}$. We equip $\mathbb{S}^d$ with the great circle (geodesic) distance, defined as
$\theta(\bm{\xi},\bm{\eta}) = \arccos(\bm{\xi}^\top \bm{\eta}) \in [0,\pi]$,
 for $\bm{\xi},\bm{\eta}\in\mathbb{S}^d$, where $\top$ denotes transpose.  
A mapping $K: \mathbb{S}^d \times \mathbb{S}^d \to \mathbb{R} $ is called a kernel if it is positive definite, that is $\sum_{i,j=1}^N a_i a_k K(\bm{\xi}_i,\bm{\xi}_j) \ge 0$, for all finite system $\{ \bm{\xi}_i \}_{i=1}^N \subset \mathbb{S}^d$ and constants $a_1, \ldots a_N \in \mathbb{R}$. This paper works with geodesically isotropic kernels, so that $K(\bm{\xi},\bm{\eta}) = \sigma^2 \psi(\theta(\bm{\xi},\bm{\eta})$ for some continuous function $\psi : [0,\pi] \to \mathbb{R}$ with $\psi(0)=1$, and for $\sigma^2 >0$. We define $\Psi_d$ as the class of continuous functions $\psi$ being the isotropic part of a kernel $K$ in $\mathbb{S}^d$. We also define $\Psi_\infty = \bigcap_{d= 1}^\infty \Psi_d$, with the strict inclusion relation
\begin{equation}\label{inclusion}
 \Psi_1 \supset \Psi_2 \supset \cdots \supset \Psi_d \supset \cdots \supset \Psi_{\infty},
\end{equation} being strict.
\cite{schoenberg} showed that a continuous mapping $\psi:[0,\pi]\rightarrow \mathbb{R}$ belongs to the class $\Psi_d$ if and only if it can be uniquely written as
\begin{equation} \label{spectral_rep}
\psi(\theta) = \sum_{m=0}^{\infty} b_{m,d} \frac{P_{m}^{(d-1)/2}(\cos \theta)}{P_{m}^{(d-1)/2}(1)}, \qquad \theta \in [0,\pi],
\end{equation}
where $P_{m}^\lambda$ denotes the $\lambda$-Gegenbauer polynomial of degree $m$ \citep[][22.2.3]{abramowitz1964handbook}, and $\{b_{m,d}\}_{n=0}^\infty$ is a probability mass sequence. \cite{schoenberg} also showed that
 $\psi$ belongs to the class $\Psi_\infty$ if and only if
\begin{equation}
\label{hilbert_sphere}
\psi(\theta) = \sum_{m=0}^{\infty} b_{m} (\cos \theta)^m, \qquad \theta \in [0,\pi],
\end{equation}
with $\{b_{m}\}_{m=0}^\infty$ being again a probability mass sequence. We follow \cite{daley-porcu} and call the sequence $ \{ b_{m,d}\}_{m=0}^{\infty}$ in (\ref{spectral_rep}) a $d$-Schoenberg sequence. Analogously, we call
$\{ b_m \}_{m=0}^{\infty}$ a Schoenberg sequence. Throughout, for a given $d$ and a given element $\psi \in \Psi_{\infty} \subset \Psi_d$, we call  $(\{ b_{m} \}_{m=0}^{\infty}, \psi )$ and $(\{ b_{m,d} \}_{m=0}^{\infty}, \psi )$ a Schoenberg and a $d$-Schoenberg pair, respectively. Section \ref{sec4} proves that these pairs are related through an operator defined therein. \\
For $d = 1$, it is true that \citep[see][]{gneiting2013}
\begin{equation} \label{d-schoenberg}
 b_{0,1} = \int_{0}^{\pi} \psi(\theta) {\rm d} \theta \qquad  \mbox{ and } \qquad  b_{m,1} = \frac{2}{\pi} \int_{0}^{\pi} \cos(m\theta) \psi(\theta) {\rm d} \theta, \mbox{ for } m \geq 1,
\end{equation}
while for $d \geq 2$ we have
\begin{equation} \label{d-schoenberg2}
 b_{m,d} = \frac{2m+d-1}{2^{3-d}\pi}\frac{\left(\Gamma\left(\frac{d-1}{2}\right)\right)^{2}}{\Gamma(d-1)}\int_{0}^{\pi} \psi(\theta) P_{m}^{(d-1)/2}(\cos \theta)  \left (\sin \theta \right )^{d-1} {\rm d} \theta.
\end{equation}
where $\Gamma(\cdot)$ denotes the gamma function \cite[6.1.1]{abramowitz1964handbook}.

The property of strict positive definiteness is described here through the members $\psi$ of the classes $\Psi_d$. By strict we mean that the inequality in the definition positive definiteness becomes strict provided the real numbers $c_{1},\ldots, c_{n}$ are not all zero; we let  $\Psi_{d}^{+}\subset \Psi_{d}$ denote the class of continuous functions $\psi$ associated with a strictly positive definite kernel on $\S^{d}.$  Arguments in \cite{schoenberg} prove that if the elements of the $d$-Schoenberg sequence $ \{ b_{m,d}\}_{m=0}^{\infty}$ in (\ref{spectral_rep})  are  positive for all  $m \ge0$ then $\psi \in \Psi_{d}^{+}. $ This simple
condition is sufficient for our purposes but the reader may consult \citep{CheMenSun03} for a careful investigation of the necessary \textsl{and} sufficient conditions.

\subsection{Harmonic Analysis on Spheres}
We now consider members $\psi$ from $\Psi_{d}^{+}$ and invoke arguments in \cite{hubbert2015spherical} to dig into an alternative view of the expansion (\ref{spectral_rep}). Specifically, we resort to Fourier expansion through spherical harmonics, that is 
\begin{equation}\label{expansion}
\psi(\boldsymbol{\xi}^{T}\boldsymbol{\eta})=\sum\limits_{m=0}^{\infty}\sum\limits_{n=1}^{N_{m,d}}\widehat{\psi}_{m} {\cal{Y}}_{m,n}(\boldsymbol\xi){\cal{Y}}_{m,n}(\boldsymbol\eta),\quad \boldsymbol\xi,\boldsymbol\eta \in \S^{d},
\end{equation}
where $\{{\cal{Y}}_{m,n}:n=1,\ldots, N_{m,d}\}$ is a real orthonormal basis for the space of spherical harmonics of degree $m$  and the collection $\{{\cal{Y}}_{m,n}:n=1,\ldots, N_{m,d}, m\ge 0\}$
forms a real orthonormal basis for $L_{2}(\S^{d}).$  In addition, $\{ \widehat{\psi}_{m} \}_{m= 0}^{\infty}$ is referred to as the sequence of spherical Fourier coefficients for $\psi \in \Psi_{d}^{+}$ and these are related to the aforementioned $d$-Schoenberg coefficients via the formula \citep[cf][Equation 1.33]{hubbert2015spherical}
\begin{equation}\label{sh-connection}
b_{m,d}=\frac{\Gamma\left(\frac{d+1}{2}\right)N_{m,d}\widehat{\psi}_{m}}{2\pi^{\frac{d+1}{2}}},
\end{equation}
where $N_{m,d}$ denotes the dimension of the space of spherical harmonics of degree $m$ given by
\begin{equation}\label{dim}
N_{0,d}=1\quad {\rm{and}}\quad  N_{m,d}=2\left(m+\frac{d-1}{2}\right)\frac{(m+d-2)!}{(d-1)!m!}.
\end{equation}
\begin{tcolorbox}[title = Remark]
The identity (\ref{sh-connection}) proves that the $d$-Schoenberg $\{b_{m,d}\}_{m=0}^{\infty}$ and the Fourier $\{ \hat{\psi}_{m}\}_{m=0}^{\infty}$ sequences are linearly related. This justify the vague terminology Fourier--Schoenberg sequences used in the introduction to this paper. 
\end{tcolorbox}

The positive spherical Fourier coefficients of $\psi \in \Psi_{d}^{+}$ decay at a polynomial rate if there exist positive constants $A_1,A_2$ and $\gamma$ such that
\begin{equation}\label{sc-decay}
\frac{A_{1}}{(1+m)^{d+\gamma}}\le  \widehat{\psi}_{m} \le \frac{A_{2}}{(1+m)^{d+\gamma}},\quad m\ge 0.
\end{equation}
Using Stirling's formula \citep[][6.1.39]{abramowitz1964handbook}, that is 
\begin{equation}\label{Stirling}
\Gamma(az+b)\sim \sqrt{2\pi}e^{-az}(az)^{az+b-\frac{1}{2}},
\end{equation}
we  deduce that
$
N_{m,d}\sim \frac{2(m+1)^{d-1}}{(d-1)!},
$
from which one can also show that there are positive constants $C_{1}, C_{2},$ independent of $m,$ such that
$
C_{1}(m+1)^{d-1}\le N_{m,d}\le C_{2}(m+1)^{d-1}, \quad m\ge 0.
$
Thus, in view of (\ref{sh-connection}), we note that  the decay condition (\ref{sc-decay}) on the spherical Fourier coefficients can be recast  in terms of the $d-$Schoenberg sequence; specifically,  there exist constants  ${\mathcal{A}}_{1},{\mathcal{A}}_{2}$ such that 
\begin{equation}\label{sch-decay}
\frac{{\mathcal{A}}_{1}}{(1+m)^{1+\gamma}}\le  b_{m,d} \le \frac{{\mathcal{A}}_{2}}{(1+m)^{1+\gamma}},\quad m\ge 0.
\end{equation}

\subsection{Special Functions}

Hypergeometric functions will feature heavily in the course of this work and so we  briefly remind the reader that a general
hypergeometric
    function is
   defined by
    \begin{equation}\label{hyper}
\pFq{p}{q}{a_{1},\cdots,a_{p}}{b_{1},\cdots,b_{q}}{z} = \sum_{n=0}^{\infty} \frac{(a_{1})_{n} \cdots (a_{p})_{n}}{(b_{1})_{n} \cdots (b_{q})_{n}} \frac{z^{n}}{n!},
\end{equation}
where 
\begin{equation}\label{poch}(c)_n := c(c+1)\cdots(c+n-1) =
\frac{\Gamma(c+n)}{\Gamma(c)}, \hspace{0.1in} {\rm{for}}\,\,\,n \geq 1,
\end{equation} denotes the Pochhammer symbol, with $(c)_0 := 1$. 
Throughout, $B$ denotes the Beta function, defined for $x>0$ and $y>0$ by
\begin{equation}\label{beta}
B(x,y)=\int_{0}^{1}t^{x-1}(1-t)^{y-1}dt = \frac{\Gamma(x)\Gamma(y)}{\Gamma(x+y)}.
\end{equation}
\section{Native Spaces on Spheres} \label{sec3}
For a given $\psi \in \Psi_{d}^{+},$ whose spherical Fourier coefficients we assume to be positive, we define the following subspace of $L_{2}(\S^{d}):$
\begin{equation}\label{native}N_{\psi}=\Bigl\{f\in L_{2}(S^{d-1}):\|f\|_{\psi}^{2}=\sum_{m=0}^{\infty}\sum_{n=1}^{N_{n,d}}\frac{|\widehat{f}_{m,n}|^{2}}{\widehat{\psi}_{m}} <\infty \Bigr\},
\end{equation}
where $\widehat{f}_{m,n}$ denote the  expansion coefficients associated to the spherical Fourier series representation  
\begin{align*}
f = \sum_{m=0}^{\infty}\sum_{n=1}^{N_{n,d}}\widehat{f}_{m,n}{\cal{Y}}_{m,n},\quad 
{\rm{where}}\,\,\,\,
\widehat{f}_{m,n}=\int_{\S^{d}}f({\bf{x}}){\cal{Y}}_{m,n}({\bf{x}})d \omega_{d}({\bf{x}}).
\end{align*} 
We observe that
$\|\cdot\|_{\psi}$ is a norm induced by the inner-product
\begin{equation}\label{native-inner}
(f,g)_{\psi}:=\sum_{m=0}^{\infty}\sum_{n=1}^{N_{n,d}}\frac{\widehat{f}_{m,n}\widehat{g}_{m,n}}{\widehat{\psi}_{m}}.
\end{equation}

We shall call $N_{\psi}$ the Native space induced by $\psi.$ We observe that if we consider the function $\psi$ whose spherical Fourier coefficients are given by
$$
\widehat{\psi}_{m}:= \frac{1}{(1+m)^{2\gamma}},
$$
then the corresponding Native space coincides with the Sobolev space of order $\gamma,$ that is
\[
W_{2}^{\gamma}(\S^{d}):=\Bigl\{f\in L_{2}(S^{d-1}):\|f\|_{ W_{2}^{\gamma}}^{2}=\sum_{m=0}^{\infty}\sum_{n=1}^{N_{n,d}}(1+m)^{2\gamma}|\widehat{f}_{m,n}|^{2} <\infty \Bigr\}.
\]

More generally, if the spherical Fourier coefficients of $\psi \in \Psi_{d}^{+}$ satisfy the decay condition (\ref{sc-decay}), or equivalently, if its $d-$Scohenberg sequence satisfies (\ref{sch-decay}), then the induced Native space $N_{\psi}$ is norm equivalent to the Sobolev space $W_{2}^{\beta}(\S^d)$ where $\beta = \frac{d+\gamma}{2}.$ That is the two spaces agree as sets and the norms are equivalent since
\[
\sqrt{A_{1}}\|f\|_{\psi}\le \|f\|_{ W_{2}^{\beta}}\le \sqrt{A_{2}}\|f\|_{\psi}.
\]
{In particular, if $N_\psi$ and $N_{\psi'}$ are norm-equivalent, then their kernel discrepancies \eqref{eq: kernel discrepancy} define the same topology on the space of probability distributions on $\mathbb{S}^d$.  }
We observe that since $\beta >\frac{d}{2}$ then, as a consequence of the Sobolev embedding theorem, the Native space is $N_{\psi}$ is continuously embedded in $C(\S^{d}),$ the space of continuous functions on $\S^{d},$ and this implies that  $N_{\psi}$ is a reproducing kernel Hilbert space. The following result concerning Native spaces is adapted from \cite{levesley2005approximation} Proposition 3.1.

\begin{lemma}\label{RKHSlem}
Let $\Psi({\bf{\xi}},{\bf{\eta}}) = \psi(\bm{\xi}^{\top} \bm{\eta})$ denote a kernel induced by  $\psi \in \Psi_{d}^{+}$, having expansion (\ref{expansion}) according to a Fourier sequence $\{ \widehat{\psi}_m\}_{m=0}^{\infty}$ of strictly positive coefficients. The corresponding Native space $N_{\psi}$ (\ref{native}) is a reproducing kernel Hilbert space with reproducing kernel $\Psi$.
\end{lemma}
\begin{tcolorbox}[title = Important Connection ]
\begin{remark}\label{onSob}
Let $\psi \in \Psi_{d}^{+}$ induce a kernel $\Psi$ as in Lemma \ref{RKHSlem}. If the spherical Fourier coefficients of $\Psi$ satisfy (\ref{sc-decay}), then $\Psi$ is a reproducing kernel for a space that is norm equivalent to the Sobolev space $W_{2}^{\beta}(\S^d)$ where $\beta = \frac{d+\gamma}{2}$.
\end{remark}
 \end{tcolorbox}

For two mappings $f,g: \mathbb{N}_0 \to \mathbb{R}$, we say that $ f(n)\sim g(n)$ if and only if 
\begin{equation}\label{Sim} \lim_{n\to \infty}\frac{f(n)}{g(n)}=1.
 \end{equation}
  A direct implication of (\ref{Sim}) is that there exists positive constants $A_{1}$ and $A_{2}$ such that
  \[
  A_{1}g(n)\le f(n)\le A_{2}g(n),\quad n\ge 0.
  \]
  In view of Remark \ref{onSob} we observe that by establishing asymptotic decay rates for various classes of covariance function we will be able to establish which order Sobolev space the covariance kernels are reproducing for.

\def\blambda{\boldsymbol{\lambda}}

\section{Quantifying Smoothness on $d$-dimensional Spheres} \label{sec4}

In this paper we will consider parametric classes of members  $\psi$ of the class $\Psi_{\infty}^{d}$.  We will access these via the two approaches described Section 1. Specifically we will either take $\psi \in \Psi_{\infty}^{+}$ as a starting point and consider its projection to $\S^{d}$ or we will take a positive definite radial kernel as a starting point and consider its restriction to $\S^{d}.$ In both cases we will derive closed form expressions for the associated $d-$Schoenberg sequences and by examining their asymptotic decay we can quantify the  smoothness properties which, in turn,  determines whether the induced Native space is norm equivalent to a Sobolev space of a certain order.

\subsection{Projecting $\Psi_{\infty}^{+}$ to $\Psi_{d}^{+}$.}

Many of the well known parametric classes in numerical analysis and statistics are defined through members of the class $\Psi_{\infty}^{+},$ i.e., the Schoenberg sequence $(b_{m})_{m=0}^{\infty}$ is known for the representation (\ref{hilbert_sphere}). This is an obstacle in the case where one wants employ such functions on a finite dimensional sphere, where one requires the $d-$Schoenberg sequence in order to quantify the smoothness of the kernel and consequently to state whether the induced Native space is norm equivalent to a Sobolev space of a certain order. In order to circumvent this we consider the following  projection operator which we will show maps the Schoenberg sequence of $\psi \in \Psi_{\infty}$ to its unique $d-$Schoenberg sequence when viewed as a member of $ \Psi_{d}..$

\begin{tcolorbox}[title = The Projection Operator]
For a Schoenberg sequence $\{b_m\}_{n=0}^{\infty}$,we define the operator $\Upsilon_d$:  
\begin{equation}
\label{operator} \forall m \in \mathbb{N}_0, \; \Upsilon_d \left( b_{m} \right ) =\frac{\sqrt{\pi}}{2^{m+d-2}\Gamma\left(\frac{d}{2}\right)}
\frac{\Gamma\left(m+d-1\right)}{m!\Gamma\left(m+\frac{d-1}{2}\right)}
\sum_{j=0}^{\infty}b_{m+2j}\frac{(m+2j)!}{j!2^{2j}\left(m+\frac{d+1}{2}\right)_{j}}
\end{equation} 
where  $(x)_j=\Gamma(x+j)/\Gamma(x)$ denotes the Pochhammer symbol \citep[][6.1.22]{abramowitz1964handbook}.
\end{tcolorbox}

\vspace{0.3cm}

 \begin{prop}[Projection Operator]
 \label{lem:d-schoen}
Let $\Upsilon_d$ be as defined through (\ref{operator}). Then, $\Upsilon_d$ maps $\Psi_{\infty}$ ($\Psi_{\infty}^{+}$) into $\Psi_{d}$ ($\Psi_{d}^{+}$). That is, let $b_{m,d}$ be defined as $b_{m,d}=\Upsilon_d(b_m)$, for $m \in \mathbb{N}$ and for $\{ b_m \}_{m=0}^{\infty}$ a Schoenberg sequence. Then, the sequence $\{ b_{m,d}\}_{m=0}^{\infty}$ is a $d$-Schoenberg sequence.
\end{prop}

\begin{proof}
According to \cite{bingham1973positive} Lemma 1,  the following identity holds

\begin{equation} \label{power}
 (\cos\theta)^m= \frac{m!\Gamma\left(\frac{d-1}{2}\right)}{2^{m}(d-2)!}\sum_{0\le 2k\le m}\frac{(m-2k+\frac{d-1}{2})(m-2k+d-2)!}{k!(m-2k)!\Gamma\left(m-k+\frac{d+1}{2}\right)}\frac{P_{m-2k}^{(d-1)/2}(\cos \theta)}{P_{m-2k}^{(d-1)/2}(1)}.
\end{equation}

This allows us to deduce that
\begin{equation} \label{link1}
\begin{aligned}
\psi(\theta) &= \sum_{m=0}^{\infty} b_{m} (\cos \theta)^m \\
&= \frac{\Gamma\left(\frac{d-1}{2}\right)}{(d-2)!}\sum_{m=0}^{\infty} \frac{b_{m}m!}{2^{m}}\sum_{0\le 2k\le m}\frac{(m-2k+\frac{d-1}{2})(m-2k+d-2)!}{k!(m-2k)!\Gamma\left(m-k+\frac{d+1}{2}\right)}\frac{P_{m-2k}^{(d-1)/2}(\cos \theta)}{P_{m-2k}^{(d-1)/2}(1)}.
\end{aligned}
\end{equation}

By inspection, the coefficient of $\frac{P_{m}^{(d-1)/2}(\cos \theta)}{P_{m}^{(d-1)/2}(1)}$ is given by

\begin{equation} \label{link-main}
\begin{aligned}
b_{m,d}&=\frac{\Gamma\left(\frac{d-1}{2}\right)\left(m+\frac{d-1}{2}\right)}{2^m}\frac{(m+d-2)!}{m!(d-2)!}\sum_{j=0}^{\infty}b_{m+2j}\frac{(m+2j)!}{j!2^{2j}\Gamma\left(m+j+\frac{d+1}{2}\right)}\\
&=\frac{\Gamma\left(\frac{d-1}{2}\right)}{\Gamma(d-1)}\frac{\left(m+\frac{d-1}{2}\right)}{\Gamma\left(m+\frac{d+1}{2}\right)}\frac{1}{2^{m}}\frac{(m+d-2)!}{m!}\sum_{j=0}^{\infty}b_{m+2j}\frac{(m+2j)!}{j!2^{2j}\left(m+\frac{d+1}{2}\right)_{j}}\\
&=\frac{\sqrt{\pi}}{2^{m+d-2}\Gamma\left(\frac{d}{2}\right)}\frac{\Gamma\left(m+d-1\right)}{m!\Gamma\left(m+\frac{d-1}{2}\right)}\sum_{j=0}^{\infty}b_{m+2j}\frac{(m+2j)!}{j!2^{2j}\left(m+\frac{d+1}{2}\right)_{j}},
\end{aligned}
\end{equation}
where the final line follows from $\Gamma(x+1)=x\Gamma(x)$ and an application the Gamma function identity \citep[][6.1.18]{abramowitz1964handbook},
\begin{equation}\label{gamma-double}
\frac{\Gamma(2z)}{\Gamma(z)}=\frac{2^{2z-1}}{\sqrt{\pi}}\Gamma\left(z+\frac{1}{2}\right).
\end{equation}
\end{proof}

\subsection{Restricting radial kernels to the sphere.}

An alternative source of members of $\Psi_{d}^{+}$ can be accessed by choosing a radial kernel $\phi$ that is known to be positive definite on $\R^{d+1}$ and then defining its restriction to the sphere. Specifically, we suppose that $d$ is a fixed space dimension and we take a parametric family $\{\phi(\cdot,\blambda), \; \blambda \in \Theta \subset \mathbb{R}^p \}$ of radial functions that are positive definite on $\R^{d+1}.$  The {\it chordal distance} on $\mathbb{S}^d$ is  connected to the {\it geodesic distance}  via
\begin{equation}
\label{chordal} d_{{\rm CH}}(\boldsymbol\xi_1,\boldsymbol\xi_2)= \|\boldsymbol\xi_1-\boldsymbol\xi_2\| 
=\sqrt{2-2\cos\left( \theta(\boldsymbol\xi_1,\boldsymbol\xi_2)\right)}
 \qquad \boldsymbol\xi_1,\boldsymbol\xi_2\in\mathbb{S}^d.
\end{equation}
Using this we define 
\begin{equation}\label{restriction}
\psi(\theta,\blambda):=\phi\left(\sqrt{2-2\cos\left( \theta(\boldsymbol\xi_1,\boldsymbol\xi_2)\right)},\blambda\right),
\end{equation}
and, by construction, this restricted family belongs to $\Psi_{d}^{+}.$  A crucial ingredient for computing the $d-$Schoenberg coefficients of the restricted family is prior knowledge of the  $d-$dimensional radial Fourier transform of $\phi.$

\begin{mydef}
Let $\phi(t)$ denote a continuous real valued function on $[0,\infty).$  The $d-$dimensional radial Fourier transform of $\phi$ is defined by \cite{stein-book}
 \begin{equation}\label{specden}
\widehat{\phi}(r)=\mathcal{F}_{d}\phi(r)\:=r^{-\frac{d-2}{2}}\int_{0}^{\infty}\phi(t)t^{\frac{d}{2}}J_{\frac{d-2}{2}}(rt)dt,\,\,\,\, r\ge0,
\end{equation}
where $J_{\nu}(\cdot)$ denotes the  Bessel function of the first kind with order $\nu.$
We note that a sufficient condition for $\widehat{\phi}(r)$ to be well-defined is that $\phi(t)t^{d-1}$ is absolutely integrable.
\end{mydef}

In this framework the $d-$Schoenberg coefficients associated to members of $\Psi_{d}^{+}$ that are defined via (\ref{restriction}) is given by the following formula  (cf. \cite{narcowich2002scattered} Theorem 4.1)
\begin{equation}
\label{linkup}
\begin{aligned}
b_{m,d}&=(2\pi)^{\frac{d+1}{2}}\kappa_{m,d}\int_{0}^{\infty}tJ_{m+\frac{d-1}{2}}^{2}(t)\widehat{\phi}(t)dt,\\
{\rm{where}}\quad \kappa_{m,d}&=\frac{\Gamma\left(\frac{d+1}{2}\right)}{2\pi^{\frac{d+1}{2}}\Gamma(d)}\frac{(2m+d-1)(m+d-2)!}{m!}.
\end{aligned}
\end{equation}

In the next section  we will use the results presented here on Hilbert space projections and on spherical restrictions to derive closed form expressions for the $d-$Schoenberg coefficients for different classes of parameterised families belonging to $\Psi_{d}^{+}.$

\section{Parameterised families and Native Sobolev spaces } \label{sec5}

For each of the  families of geodesically isotropic kernels presented in this section we will provide: 
\begin{description}
\item[{\bf 1.}] A closed form expression of their $d-$Schoenberg sequence. 
\item[{\bf 2.}] The asymptotic rate of decay of their $d-$Schoenberg sequence.  
\item[{\bf 3}.] The Native Sobolev space for which the kernels are reproducing for.
\end{description}

\subsection{ The {\bf Mat{\'e}rn} Class of functions }

For $\nu,\alpha>0,$ the Mat{\'e}rn class of functions are  well-known positive definite radial kernels defined
as \citep{stein-book}
\[
 {\cal M}_{\nu, \alpha} (r) = \frac{2^{1-\nu}}{\Gamma(\nu)} \left ( \frac{r}{\alpha} \right )^{\nu} {\cal K}_{\nu}\left ( \frac{r}{\alpha} \right ), 
\]
 with ${\cal K}_{\nu}$ a modified Bessel function of the second kind of order $\nu$ \citep{abramowitz1964handbook}[9.6.22]. The Mat{\'e}rn class has been especially popular in spatial statistics after \cite{stein-book}. We consider the restriction of this family to the sphere which we define as
 \begin{equation}\label{sphMatern}
 \psi_{{\cal M}}(\theta, \blambda):= {\cal M}_{\nu, \alpha}(\sqrt{2 - 2 \cos (\theta)}), \,\,\, {\rm{for}}\,\,\, \blambda=(\nu,\alpha)^{\top} \in (0,\infty)^2.
 \end{equation}
  
 In order to apply (\ref{linkup}) and derive the $d-$Schoenberg coefficients associated to the family $\psi_{{\cal M}}(\theta, \blambda)$ we require prior knowledge of the radial Fourier transform of $ {\cal M}_{\nu, \alpha} (r).$ This is given in the following result.
  \begin{lemma}
 \label{lem:FTMatern}
Let $\alpha$ and $\nu$  be positive real numbers. The $d-$dimensional radial Fourier transform of the  Mat{\'e}rn kernel  ${\cal M}_{\nu, \alpha} (r)$  is given by
\begin{equation}\label{MatFT}
\widehat{{\cal M}_{\nu, \alpha} }(r) = \mathcal{F}_{d}{\cal M}_{\nu, \alpha}(r)
=\frac{2^{\frac{d}{2}}\Gamma\left(\nu+\frac{d}{2}\right)}{\alpha^{2\nu}\Gamma\left(\nu\right)}
\frac{1}{\left(\frac{1}{\alpha^{2}}+r^2\right)^{\nu+\frac{d}{2}}}.
\end{equation}
\end{lemma}

\begin{proof}

Using (\ref{specden}) the $d-$dimensional radial Fourier transform of ${\cal M}_{\nu, \alpha} (r)$ is given by
\[
\begin{aligned}
\widehat{{\cal M}_{\nu, \alpha} }(r)& =r^{-\left(\frac{d-2}{2}\right)}\int_{0}^{\infty}{\cal M}_{\nu, \alpha}(t)t^{\frac{d}{2}}J_{\frac{d-2}{2}}(rt)dt \\
&=\frac{1}{r^{\frac{d-2}{2}}}\frac{2^{1-\nu}}{\Gamma(\nu)} \frac{1}{\alpha^{\nu}}\int_{0}^{\infty}t^{\nu+\frac{d+2}{2}-1}{\cal K}_{\nu}\left ( \frac{t}{\alpha} \right )J_{\frac{d-2}{2}}(rt)dt.
\end{aligned}
\]

The following identity is taken from \cite{prudnikov1981integralsv2} eq 2.16.21.1
\[
\begin{aligned}
&\int_{0}^{\infty}t^{\beta-1}{\cal K}_{\nu}\left(ct  \right)J_{\mu}(bt)dt\\
&=2^{\beta-1}b^{\mu}\left(\frac{1}{c}\right)^{\beta+\mu}\frac{\Gamma\left(\frac{\beta+\mu+\nu}{2}\right)\Gamma\left(\frac{\beta+\mu-\nu}{2}\right)}{\Gamma(\mu+1)}
\pFq{2}{1}{\frac{\beta+\mu+\nu}{2},\frac{\beta+\mu-\nu}{2}}{\mu+1}{-\left(\frac{b}{c}\right)^2}
\end{aligned}
\]
%

   Setting the following parameters:
 \[
 b=r, \,\,\, c = \frac{1}{\alpha}\,\,\,\beta=\nu+\frac{d+2}{2}\,\,\,\mu=\frac{d-2}{2}\rightarrow \beta+\mu+\nu = 2\nu+d,\,\,{\rm{and}}\,\,\, \beta+\mu-\nu = d,
 \]
 leads us to deduce that
 \begin{equation}
 \label{FT_hyper}
 \widehat{{\cal M}_{\nu, \alpha} }(r)=2^{\frac{d}{2}}\alpha^{d}\frac{\Gamma\left(\nu+\frac{d}{2}\right)}{\Gamma(\nu)}
 \pFq{2}{1}{\nu+\frac{d}{2},\frac{d}{2}}{\frac{d}{2}}{-\alpha^2 r^2}
 \end{equation}
The following identity is taken from  \citep[][15.1.8]{abramowitz1964handbook} ${}_{2}F_{1}(a,b;b;z) =(1-z)^{-a}.$ Inserting this into the above leads to

  \[
  \widehat{{\cal M}_{\nu, \alpha} }(r) = \frac{2^{\frac{d}{2}}\alpha^{d}\Gamma\left(\nu+\frac{d}{2}\right)}{\Gamma\left(\nu\right)}\frac{1}{(1+\alpha^{2}r^2)^{\nu+\frac{d}{2}}},
  \]
  as required. \end{proof}

Equipped with the expression for $ \widehat{{\cal M}_{\nu, \alpha} }(r) $ we can now employ (\ref{linkup}) to derive  the $d-$Schoenberg coefficients and investigate their asymptotic decay rate. This leads us to the following result. 

    \begin{prop}
 \label{lem:FTMaternAsymp}
Let $ \blambda=(\nu,\alpha)^{\top} \in (0,\infty)^2$ and consider the spherical Mat{\'e}rn  family
 $\psi_{{\cal M}}(\theta, \blambda)$ given by (\ref{sphMatern}). Then we have that

\begin{enumerate}
\item  The $d-$Schoenberg coefficients are given by
\begin{equation}\label{dMatern}
\begin{aligned}
b_{m,d,{\cal M}}(\blambda) &=(2\pi)^{\frac{d}{2}}\frac{2^{\frac{d}{2}}\Gamma\left(\nu+\frac{d}{2}\right)}{\Gamma(\nu)\alpha^{2\nu}}
\frac{\Gamma(m-\nu)\kappa_{m,d}}{\Gamma(m+\nu+d)}\,\,\pFq{1}{2}{\nu+\frac{d}{2}}{\nu+1-m, m+\nu+d}{\frac{1}{\alpha^2}}\\
&+\frac{2\pi^{\frac{d+3}{2}}}{\Gamma(\nu)}\frac{(-1)^{m}\kappa_{m,d}}{\Gamma(m+1-\nu)\Gamma\left(m+\frac{d+1}{2}\right)(2\alpha)^{m}}
\pFq{1}{2}{m+\frac{d}{2}}{m-\nu+1, 2m+d}{\frac{1}{\alpha^2}}.
\end{aligned}
\end{equation}
\item Further,
\[
b_{m,d,{\cal M}}(\blambda)\sim  \frac{2}{\alpha^{2\nu}}\frac{\Gamma\left(\nu+\frac{d}{2}\right)}{\Gamma\left(\nu\right)\Gamma\left(\frac{d}{2}\right)}\frac{1}{m^{1+2\nu}}.
\]
\item  The native space $N_{\psi_{{\cal M}}}$ associated with the Mat{\'e}rn kernel is a reproducing kernel Hilbert space with reproducing kernel $\Psi_{{\cal M}} ({\bf{x}},{\bf{y}})= \psi_{{\cal M}} ({\bf{x}}^{T}{\bf{y}}; \blambda).$ Furthermore, $N_{\psi_{{\cal M}}}$ is norm equivalent to the Sobolev space $W_{2}^{\beta}(\S^{d})$ where $\beta = \nu+\frac{d}{2}.$
\end{enumerate}
\end{prop}

\begin{proof} See Appendix A.
\end{proof}

\subsection{ The ${\cal F}$-Class of functions }

Recently, \cite{alegria2021f} have proposed the ${\cal F}$ family by
\begin{equation}
\label{cov_2f1}
\psi_{{\cal F}}(\theta, \blambda) = \frac{B(\alpha,\nu+\tau)}{B(\alpha,\nu)} {}_{2}F_{1}(\tau,\alpha,\alpha+\nu+\tau;\cos \theta),\,\,\,\blambda=(\tau,\alpha,\nu)^{\top} \in (0,\infty)^3 ,
\end{equation}
where $\theta \in [0,\pi],$ $B$ is the Beta function defined by (\ref{beta}) and ${}_{2}F_{1}$ is defined through (\ref{hyper}).
  \begin{prop}
 \label{FHilbert}
 Let $\blambda = \in  (\tau,\alpha,\nu)^{\top}\in\mathbb{R}_+^3$ denote the parameter vector associated with (\ref{cov_2f1}). The correspnding Schoenberg sequence is given by
\begin{equation}
\label{schoenberg_coef}
b_{m,{\cal F}}(\blambda)
 = \frac{B(\alpha,\nu+\tau)}{B(\alpha,\nu)}\frac{(\tau)_{m}(\alpha)_{m}}{(\alpha+\nu+\tau)_{m}m!}>0\quad m\ge0,
\end{equation}
and consequently $\psi_{{\cal F}}(\theta, \blambda)$ belong to the class $\Psi_{\infty}^{+}.$
\end{prop}
\begin{proof} This follows from the definition of the hypergeometric ${}_{2}F_{1}$  (\ref{hyper}). The coefficients are positive since the parameters of $ \blambda$ are positive.
\end{proof}

Equipped with the expression for $ b_{m,{\cal F}}(\blambda) $ we can now employ the projection operator (\ref{operator}) to derive  the $d-$Schoenberg coefficients and investigate their asymptotic decay rate. This leads us to the following result. 

\begin{prop}\label{Fresult}
 Let $$ \Big \{  \big (\{ b_{m,{\cal F}}(\blambda) \}_{m=0}^{\infty}, \psi_{{\cal F}}(\theta, \blambda) \big ); \; \; \blambda = \in  (\tau,\alpha,\nu)^{\top}\in\mathbb{R}_+^3 \Big \}$$ be the Schoenberg pair for the ${\cal F}-$family as given in Proposition \ref{FHilbert}. Then, 
\begin{enumerate}
\item The $d$-Schoenberg sequence $\{ b_{m,d,{\cal F}}(\blambda)\}_{m=0}^{\infty}$ is uniquely determined through
 \begin{equation}
\label{F-d-coef}
b_{m,d,{\cal F}}(\blambda)=
C_{m,d}(\tau,\alpha,\nu)\pFq{4}{3}{\frac{\alpha+m}{2},\frac{\alpha+m+1}{2},\frac{\tau+m}{2},\frac{\tau+m+1}{2}}{\frac{\alpha+\nu+\tau+m}{2},\frac{\alpha+\nu+\tau+m+1}{2},
m+\frac{d+1}{2}}{1},
\end{equation}
where
 \begin{equation}
\label{Constcoef}
C_{m,d}(\tau,\alpha,\nu)=\frac{b_m(\tau, \alpha, \nu)}{2^{m+d-2}}\frac{\Gamma(m+d-1)}{\Gamma\left(m+\frac{d-1}{2}\right)}\frac{\sqrt{\pi}}{\Gamma\left(\frac{d}{2}\right)}.
\end{equation}
\item It is true that
\begin{equation}\label{precise}
 b_{m,d,{\cal F}}(\blambda)\sim \frac{\Gamma(\nu+\alpha)\Gamma(\nu+\tau)}{\Gamma(\alpha)\Gamma(\nu)\Gamma(\tau)}
\frac{2^{\nu+1}\Gamma\left(\frac{d}{2}+\nu\right)}{\Gamma\left(\frac{d}{2}\right)}\frac{1}{m^{1+2\nu}}.
\end{equation}
\item The native space $N_{\psi_{\mathcal{F}}(\blambda)}$ associated with  $\mathcal{F}_{\tau,\alpha,\nu} \in \Psi_{d}^{+}$ is a reproducing kernel Hilbert space with reproducing kernel ${\mathcal{F}}_{\tau,\alpha,\nu}({\bf{x}}^{T}{\bf{y}}).$ Furthermore, $N_{\psi_{\mathcal{F}}(\blambda)}$ is norm equivalent to the Sobolev space $W_{2}^{\beta}(\S^{d})$ where $\beta = \nu+\frac{d}{2}.$
\end{enumerate}
\end{prop}

\begin{proof}
See Appendix
\end{proof}

\subsection{The {\bf Generalised Wendland}  family}

The  Generalised Wendland family of radial functions are defined as
\begin{equation}\label{gwend}
\begin{aligned}
&{\cal W}_{\nu,\alpha,\epsilon}(r) =
\frac{1}{B(2\alpha,\nu+1)}\int_{\epsilon r }^{1} {\cal W}_{\nu,0,1}(t) \, t \,
\left(t^2-(\epsilon r)^2 \right)^{\alpha-1} \mathrm{d}t \\
&=\frac{B(\alpha,\nu+1)}{2^{\nu+1}B(2\alpha,\nu+1)}\left(1-(\epsilon r)^2\right)^{\nu+\alpha}
\pFq{2}{1}{\frac{\nu}{2},\frac{\nu+1}{2}}{\nu+\alpha+1}{1-(\epsilon r)^2}
\quad
r \in \left[0,\frac{1}{\epsilon}\right],
\end{aligned}
\end{equation} where $\nu > 0,$ $\alpha>0$ and the constant multiplier is chosen to
is chosen to ensure $\phi_{\nu,\alpha, \epsilon}(0)=1.$ Here,
${\cal W}_{\nu,0,\epsilon}(r):=(1-\epsilon r)_{+}^{\nu}$, with $(a)_+$ denoting the positive part of the real number $a$. We note that the functions in this family are compactly supported, where the parameter $\epsilon$ controls the size of the supporting interval.

Arguments in \cite{Chernih2014}  show that ${\cal W}_{\nu,\alpha,\epsilon}(r)$ is positive definite on $\R^{d+1}$  provided that $\nu\geq \frac{d+2}{2}+\alpha$ and so, under these conditions, we can define their restriction to the sphere  $\S^{d}$ via
\begin{equation}\label{sphWend}
\psi_{{\cal W}}(\theta, \blambda)={\cal W}_{\nu,\alpha,\epsilon}(\sqrt{2-2 \cos(\theta)}), \qquad \blambda =(\nu,\alpha,\epsilon)^{\top} \in \R^{3}_{+},
\end{equation}
where $\theta \in [0,\pi].$ By construction $\psi_{{\cal W}}(\theta, \blambda)$ belong to $\Psi_d^{+}$ provided $\nu\geq \frac{d+2}{2}+\alpha$. The properties of these restricted functions have been investigated in detail in 
\cite{Hubbert2021}  and these findings are summarised in the following result.

    \begin{teo}
 \label{WendMain}
 Let $$ \Big \{  \big (\{ b_{m,d,{\cal W}}(\blambda) \}_{m=0}^{\infty}, \psi_{{\cal W}}(\theta, \blambda) \big ); \; \; \blambda = (\alpha,\nu,\epsilon)^{\top} \in \mathbb{R}_+^3 \Big \}$$ be the $d-$Schoenberg pair for the generalised Wendland family  (\ref{sphWend}). Then,  
 \begin{enumerate}
 \item It is true that 
 \begin{equation}\label{Wend-d_Schoen}
 \begin{aligned}
 b_{m,d,{\cal W}}(\blambda) &= \frac{2\Gamma(2\alpha+\nu+1)}{\Gamma(2\alpha+\nu+1+d)B\left(\alpha+\frac{1}{2},\frac{d}{2}\right)}\frac{1}{\epsilon^{d}}\\
 &\times
 \frac{\left(m+\frac{d-1}{2}\right)(m+d-2)!}{m!}\pFq{3}{2}{- \left(m+\frac{d-2}{2}\right), m+\frac{d}{2}, \frac{d+1}{2}+\alpha}{\frac{d+1}{2}+\alpha+\frac{\nu}{2},\frac{d+1}{2}+\alpha+\frac{\nu+1}{2}}{
\frac{1}{4\epsilon^2}}. 
\end{aligned}
\end{equation}
\item There exist  two positive constants ${\mathcal{A}}_{1}<{\mathcal{A}}_{2}$ such that
\begin{equation}\label{tightW}
\frac{{\mathcal{A}}_{1}\epsilon^{2\alpha+1}}{(1+m)^{2+2\alpha}}\le b_{m,d,{\cal W}}(\blambda)
\le \frac{{\mathcal{A}}_{2}\epsilon^{2\alpha+1}}{(1+m)^{2+2\alpha}}.
\end{equation}
\item The native space $N_{\psi_{{\cal W}}}$ associated to $\psi_{{\cal W}}$  is a reproducing kernel Hilbert space with reproducing kernel $\psi_{{\cal W}}({\bf{x}}^{T}{\bf{y}},\blambda).$ Furthermore, $N_{\psi_{{\cal W}}}$ is norm equivalent to the Sobolev space $W_{2}^{\beta}(\S^{d})$ where $\beta = \alpha+\frac{1}{2}+\frac{d}{2}.$
\end{enumerate}
\end{teo}

\begin{proof}
The expression for the $d-$Schoenberg coefficients of the generalised Wendland functions can be derived from the closed form expression of their spherical Fourier coefficients as computed in \cite{Hubbert2021} (Theorem 4.7), together with (\ref{sh-connection}). The tight asymptotic bounds follow from  \cite{Hubbert2021} (Theorem 5.8). Lemma \ref{RKHSlem} shows that the  native space  $N_{\psi_{{\cal W}}}$ possesses the stated  reproducing kernel properties.  The norm equivalence of    $N_{\psi_{{\cal W}}}$   to the Sobolev space of order $ \alpha+\frac{1}{2}+\frac{d}{2}$ follows from Remark \ref{onSob} and the decay condition (\ref{tightW}) on $b_{m,d,{\cal W}}(\blambda).$
\end{proof}

\section{Discussion} \label{secdisc}
This paper provides new tools that allow for a precise identification of the Sobolev space associated with a given kernel defined over a $d$-dimensional hypersphere. 
An immediate consequence of our results is an improved understanding of kernel cubature, since once a Sobolev space associated to a kernel has been identified one can determine the rate of convergence of the associated discrepancy (i.e. the worst-case cubature error), using for example the techniques described in \citet{krieg2021function} and the references therein.
Our results also extend the applicability of the Riemann--Stein cubature method of \cite{barp2022riemann}, used to accelerate posterior computation in the Bayesian context, since this method requires the Sobolev space associated with a kernel to be precisely identified. 

Some further extensions of our results might be possible at the expense of additional effort. 
For instance, we are confident that the extension of the present work to the case of compact two-point homogeneous spaces would apply {\em mutatis mutandis} by replacing the Gegenbauer polynomials in the Schoenberg expansion with Jacobi polynomials. Some other extensions might be more challenging. For instance, we are unaware at the moment of how to characterise Sobolev cases on hyperspheres when the kernels is not isotropic, but axially symmetric only \citep{jones}. Another interesting case would be that of product spaces involing the hypersphere with any locally compact group. Finally, we would like to mention that the recent {\em tour de force} by \cite{wynne2022spectral} opens for considering the present work in the direction of operator valued kernels.

\section*{Acknowledgement}

CJO was supported by EPSRC [EP/W019590/1].
MG was supported by a Royal Academy of Engineering Research Chair and EPSRC [EP/T000414/1, EP/R018413/2, EP/P020720/2, EP/R034710/1, EP/R004889/1].


\section{} \label{app}

\section*{Results associated to the Mat{\'e}rn kernel}

Here we present the proof of the 3 statements of Proposition  \ref{lem:FTMaternAsymp} associated to the Mat{\'e}rn kernel. 

\subsection*{Proposition \ref{lem:FTMaternAsymp} Statement 1.}

\begin{proof}

Using (\ref{linkup}) and (\ref{MatFT}) we can write
\begin{equation}\label{D1}
\begin{aligned}
b_{m,d,{\cal M}}(\blambda)&=(2\pi)^{\frac{d+1}{2}}\kappa_{m,d}\int_{0}^{\infty}tJ_{m+\frac{d-1}{2}}^{2}(t)   \mathcal{F}_{d+1}{\cal M}_{\nu, \alpha}(t)dt\\
&=(2\pi)^{\frac{d+1}{2}}\kappa_{m,d}\frac{2^{\frac{d+1}{2}}\alpha^{d+1}\Gamma\left(\nu+\frac{d+1}{2}\right)}{\Gamma\left(\nu\right)} \int_{0}^{\infty} \frac{tJ_{m+\frac{d-1}{2}}^{2}(t)}{(1+\alpha^{2}t^2)^{\nu+\frac{d+1}{2}}}dt
\\& =\pi^{\frac{d+1}{2}}\kappa_{m,d}\frac{2^{d+1}\Gamma\left(\nu+\frac{d+1}{2}\right)}{\Gamma\left(\nu\right)\alpha^{2\nu}} \int_{0}^{\infty} \frac{tJ_{m+\frac{d-1}{2}}^{2}(t)}{\left(\frac{1}{\alpha^2}+t^2\right)^{\nu+\frac{d+1}{2}}}dt.
\end{aligned}
\end{equation}

The following formula is adapted from \cite{prudnikov1981integralsv2} 2.12.32.10

\[
\begin{aligned}
&\int_{0}^{\infty}\frac{t^{\beta-1}J_{\mu}^{2}(t) }{(z^2+t^2)^{\rho}}dt\\
&
=\frac{1}{2^{2\rho+1-\beta}}\frac{\Gamma\left(\mu-\rho+\frac{\beta}{2}\right)\Gamma(1+2\rho-\beta)}{\Gamma\left(\rho+1-\frac{\beta}{2}\right)^{2}\Gamma\left(\mu+\rho+1-\frac{\beta}{2}\right)}\pFq{2}{3}{\rho+\frac{1-\beta}{2},\rho}{\rho+1-\mu-\frac{\beta}{2}, \rho+1+\mu-\frac{\beta}{2}, 1+\rho-\frac{\beta}{2}}{z^2}\\
&+ \frac{z^{2\mu+\beta-2\rho}}{2^{2\mu+1}}\frac{\Gamma\left(\rho-\mu-\frac{\beta}{2}\right)\Gamma\left(\mu+\frac{\beta}{2}\right)}{\Gamma(\rho)\Gamma(\mu+1)^{2}}\pFq{2}{3}{\mu+\frac{1}{2},\frac{\beta}{2}+\mu}{1-\rho+\frac{\beta}{2}+\mu, \mu+1, 2\mu+1}{z^2},
\end{aligned}
\]
and holds for $\beta+2\mu>0,$ and $\beta-2\rho<1.$ Setting $\beta =2,$ $z=\frac{1}{\alpha}$, $\mu = m+\frac{d-1}{2}$ and $\rho =\nu+\frac{d+1}{2}$ (where $\nu \notin \Z_{+})$ yields

\[
\begin{aligned}
&\int_{0}^{\infty} \frac{tJ_{m+\frac{d-1}{2}}^{2}(t)}{\left(\frac{1}{\alpha^2}+t^2\right)^{\nu+\frac{d+1}{2}}}dt\\
&=\frac{1}{2^{2\nu+d}}\frac{\Gamma(m-\nu)\Gamma\left(2\nu+d\right)}{\Gamma\left(\nu+\frac{d+1}{2}\right)^2 \Gamma(m+\nu+d)}
\pFq{1}{2}{\nu+\frac{d}{2}}{\nu+1-m, m+\nu+d}{\frac{1}{\alpha^2}}\\
&+\frac{1}{(2\alpha)^{2m}}\frac{\alpha^{2\nu}}{2^{d}}\frac{\Gamma(\nu-m)}{\Gamma\left(\nu+\frac{d+1}{2}\right)\Gamma\left(m+\frac{d+1}{2}\right)}\pFq{1}{2}{m+\frac{d}{2}}{m-\nu+1, 2m+d}{\frac{1}{\alpha^2}}.
\end{aligned}
\]
We remark that the ${}_{2}F_{3}$ hypergeometric functions from the formula collapse to ${}_{1}F_{2}$ hypergeometric functions in the application above, this is due to a repeated parameter appearing in both case; $\nu+\frac{d+1}{2}$ in the first instance and $m+\frac{d+1}{2}$ in the second.  With this integral computed we can conclude that
\[
\begin{aligned}
b_{m,d,{\cal M}}(\blambda)&=\frac{2\pi^{\frac{d+1}{2}}}{\Gamma(\nu)(2\alpha)^{2\nu}}\frac{\Gamma(2\nu+d)}{\Gamma\left(\nu+\frac{d+1}{2}\right)}\frac{\Gamma(m-\nu)\kappa_{n,d}}{\Gamma(m+\nu+d)}\,\,\pFq{1}{2}{\nu+\frac{d}{2}}{\nu+1-m, m+\nu+d}{\frac{1}{\alpha^2}}\\
&+\frac{2\pi^{\frac{d+1}{2}}}{\Gamma(\nu)}\frac{\Gamma(\nu-m)}{\Gamma\left(m+\frac{d+1}{2}\right)}\frac{\kappa_{m,d}}{(2\alpha)^{m}}\pFq{1}{2}{m+\frac{d}{2}}{m-\nu+1, 2m+d}{\frac{1}{\alpha^2}}.
\end{aligned}
\]

Applying (\ref{gamma-double}) and the reflection formula for the Gamma function  \citep[][6.1.17]{abramowitz1964handbook} we can write this as 

\[
\begin{aligned}
b_{m,d,{\cal M}}(\blambda)&=(2\pi)^{\frac{d}{2}}\frac{2^{\frac{d}{2}}\Gamma\left(\nu+\frac{d}{2}\right)}{\Gamma(\nu)\alpha^{2\nu}}
\frac{\Gamma(m-\nu)\kappa_{m,d}}{\Gamma(m+\nu+d)}\,\,\pFq{1}{2}{\nu+\frac{d}{2}}{\nu+1-m, m+\nu+d}{\frac{1}{\alpha^2}}\\
&+\frac{2\pi^{\frac{d+3}{2}}}{\Gamma(\nu)}\frac{(-1)^{m}\kappa_{m,d}}{\Gamma(m+1-\nu)\Gamma\left(m+\frac{d+1}{2}\right)(2\alpha)^{m}}
\pFq{1}{2}{m+\frac{d}{2}}{m-\nu+1, 2m+d}{\frac{1}{\alpha^2}},
\end{aligned}
\]
as required.
\end{proof}

\subsection*{Proposition \ref{lem:FTMaternAsymp} Statements 2 and 3.}
\begin{proof}

The following result provides the large parameter asymptotic behaviour of a ${}_{1}F_{2}$ of the same style as the first term in (\ref{dMatern}), it is adapted from \cite{luke1969special} 7.3(11) 

\begin{equation}\label{as1}
\pFq{1}{2}{a}{b-m, c+m}{z}=1+\sum_{j=1}^{n}\frac{(a)_{j}z^{j}}{(b-m)_{j}(c+m)_{j}j!}+O\left(\frac{1}{m^{2n+2}}\right),
\end{equation}

where $m-b \ne 0,1,2\ldots.$ The next result is adapted from \cite{luke1969special} 7.3(8) and provides the large parameter asymptotic behaviour of a ${}_{1}F_{2}$ of the same style as the second  term in (\ref{dMatern})

\begin{equation}\label{as2}
\pFq{1}{2}{\alpha+m}{\beta+m,2m+\lambda+1}{z}=1+\sum_{j=1}^{n}\frac{(\alpha+m)_{j}z^{j}}{(\beta+m)_{j}(2m+\lambda+1)_{j}j!}
+O\left(\frac{1}{m^{n}}\right).
\end{equation}

Applying Stirling's formula (\ref{Stirling}) we can deduce that the constant $\kappa_{m,d}$ (\ref{linkup}) grows asymptotically as
\begin{equation}\label{kappaas}
\kappa_{m,d}\sim \frac{m^{d-1}}{2^{d-1}\pi^{\frac{d}{2}}\Gamma\left(\frac{d}{2}\right)}.
\end{equation}

In addition, Stirling's formula also gives the following asymptotics for the Gamma functions involving $m$ appearing in (\ref{dMatern})
\begin{equation}\label{Gammaas}
\frac{\Gamma(m-\nu)}{\Gamma(m+\nu+d}\sim \frac{1}{m^{2\nu+d}}\quad {\rm{and}}\quad
\frac{1}{\Gamma(m+1-\nu)\Gamma\left(m+\frac{d+1}{2}\right)}\sim \frac{1}{2\pi}\frac{1}{m^{\frac{d+1}{2}-\nu}}\left(\frac{e}{m}\right)^{2m}.
\end{equation}

Using these asymptotic components in (\ref{dMatern}) we can deduce that, for large $m,$ we have 
\[
\begin{aligned}
b_{m,d,{\cal M}}(\blambda)  &\sim \frac{2}{\alpha^{2\nu}}\frac{\Gamma\left(\nu+\frac{d}{2}\right)}{\Gamma\left(\nu\right)\Gamma\left(\frac{d}{2}\right)}\frac{1}{m^{1+2\nu}}\Bigl[1+O\left(\frac{1}{m^2}\right)\Bigr]+\frac{\sqrt{\pi}}{2^{d-1}}\frac{(-1)^{m}m^{\nu-\frac{d}{2}}}{\Gamma\left(\frac{d}{2}\right)}\left(\frac{e^{2}}{2\alpha m^2}\right)^{m}\Bigl[1+O\left(\frac{1}{m}\right)\Bigr].
\end{aligned}
\]
Clearly the second component of the above asymptotic decays at an exponentially fast rate and so, to leading order, we have
\[
\begin{aligned}
b_{m,d,{\cal M}}(\blambda) &\sim \frac{2}{\alpha^{2\nu}}\frac{\Gamma\left(\nu+\frac{d}{2}\right)}{\Gamma\left(\nu\right)\Gamma\left(\frac{d}{2}\right)}\frac{1}{m^{1+2\nu}}.
\end{aligned}
\]
Lemma \ref{RKHSlem} shows that the  native space  $N_{\psi_{{\cal M}}}$ possesses the stated  reproducing kernel properties.  The norm equivalence of    $N_{\psi_{{\cal M}}}$   to the Sobolev space of order $ \nu+\frac{d}{2}$ follows from Remark \ref{onSob} and the established asymptotic decay rate of $b_{m,d,{\cal M}}(\blambda).$
\end{proof}

\section*{Results associated to the {\cal F} family}

Here we present the proof of the 3 statements of Proposition  \ref{Fresult} associated to the {\cal F} family.

\subsection*{Proof of Proposition \ref{Fresult} Statement 1}
\begin{proof}
For brevity we shall write $b_{m,d}$ for $b_{m,{\cal F}}(\blambda)$ in this proof. Applying (\ref{link-main}) we have
\[
b_{m,d}=\frac{B(\alpha,\nu+\tau)}{B(\alpha,\nu)}\frac{\sqrt{\pi}}{2^{m+d-2}\Gamma\left(\frac{d}{2}\right)}
\frac{\Gamma\left(m+d-1\right)}{m!\Gamma\left(m+\frac{d-1}{2}\right)}\sum_{j=0}^{\infty}\frac{(\tau)_{m+2j}(\alpha)_{m+2j}}{(\alpha+\nu+\tau)_{m+2j}j!2^{2j}\left(m+\frac{d+1}{2}\right)_{j}}.
\]
The following identities are taken from   \cite{prudnikov1981integrals} Appendix 1.6

\begin{equation}\label{double-Poch}
 (x)_{2j}=2^{2j}\left(\frac{x}{2}\right)_{j}\left(\frac{x+1}{2}\right)_{j}\quad {\rm{and}}\quad (x)_{m+2j}=(x)_{m}(x+m)_{2j}.
\end{equation}

Applying these we can show that
\[
\frac{B(\alpha,\nu+\tau)}{B(\alpha,\nu)}\frac{(\tau)_{m+2j}(\alpha)_{m+2j}}{(\alpha+\nu+\tau)_{m+2j}m!}=b_m(\tau, \alpha, \nu)\frac{2^{2j}\left(\frac{\alpha+m}{2}\right)_{j}\left(\frac{\alpha+m+1}{2}\right)_{j}\left(\frac{\tau+m}{2}\right)_{j}\left(\frac{\tau+m+1}{2}\right)_{j}}{\left(\frac{\alpha+\nu+\tau+m}{2}\right)_{j}\left(\frac{\alpha+\nu+\tau+m+1}{2}\right)_{j}}
\]
and so
\[
\begin{aligned}
b_{m,d}&=\frac{b_m(\tau, \alpha, \nu)}{2^{m+d-2}}\frac{\Gamma(m+d-1)}{\Gamma\left(m+\frac{d-1}{2}\right)}\frac{\sqrt{\pi}}{\Gamma\left(\frac{d}{2}\right)}
\sum_{j=0}^{\infty}
\frac{
\left(\frac{\alpha+m}{2}\right)_{j}
\left(\frac{\alpha+m+1}{2}\right)_{j}
\left(\frac{\tau+m}{2}\right)_{j}
\left(\frac{\tau+m+1}{2}\right)_{j}}
{\left(\frac{\alpha+\nu+\tau+m}{2}\right)_{j}
\left(\frac{\alpha+\nu+\tau+m+1}{2}\right)_{j}\left(
m+\frac{d+1}{2}\right)_{j}j!}\\
&=\frac{b_m(\tau, \alpha, \nu)}{2^{m+d-2}}\frac{\Gamma(m+d-1)}{\Gamma\left(m+\frac{d-1}{2}\right)}\frac{\sqrt{\pi}}{\Gamma\left(\frac{d}{2}\right)}
\pFq{4}{3}{\frac{\alpha+m}{2},\frac{\alpha+m+1}{2},\frac{\tau+m}{2},\frac{\tau+m+1}{2}}{\frac{\alpha+\nu+\tau+m}{2},\frac{\alpha+\nu+\tau+m+1}{2},
m+\frac{d+1}{2}}{1},
\end{aligned}
\]


where, in the final line, we recognise the infinite series as  the $ {}_{4}F_{3}$ hypergeometric function. 
\end{proof}

\subsection*{Proposition \ref{Fresult} Statements 2 and 3.}

\begin{proof}

We begin by examining the asymptotic decay of the multiple of the  $ {}_{4}F_{3}$ hypergeometric function from (\ref{F-d-coef}), i.e., we consider
\[
\begin{aligned}
C_{m,d}(\tau,\alpha,\nu)&=\frac{\sqrt{\pi}}{\Gamma\left(\frac{d}{2}\right)2^{m+d-2}}\frac{\Gamma(m+d-1)}{\Gamma\left(m+\frac{d-1}{2}\right)}\frac{B(\alpha,\nu+\tau)}{B(\alpha,\nu)}\frac{(\tau)_{m}(\alpha)_{m}}{(\alpha+\nu+\tau)_{m}m!}\\
&=\frac{\Gamma(\nu+\alpha)\Gamma(\nu+\tau)}{\Gamma(\alpha)\Gamma(\nu)\Gamma(\tau)}
\frac{\sqrt{\pi}}{\Gamma\left(\frac{d}{2}\right)2^{m+d-2}}
\frac{\Gamma(m+d-1)}{\Gamma\left(m+\frac{d-1}{2}\right) }\frac{\Gamma(m+\tau)\Gamma(m+\alpha)}{\Gamma(m+\alpha+\nu+\tau)\Gamma(m+1)}.
\end{aligned}
\]

In the case where $m$ is large  we can apply Stirling's asymptotic formula (\ref{Stirling}) to deduce that
\begin{equation}\label{Const_decay}
C_{m,d}(\tau,\alpha,\nu)\sim \frac{\Gamma(\nu+\alpha)\Gamma(\nu+\tau)}{\Gamma(\alpha)\Gamma(\nu)\Gamma(\tau)}
\frac{\sqrt{\pi}}{\Gamma\left(\frac{d}{2}\right)2^{n+d-2}} \frac{n^{\frac{d-1}{2}}}{m^{1+\nu}}.
\end{equation}
We now move on to the asymptotic decay of the  $ {}_{4}F_{3}$ hypergeometric function from (\ref{F-d-coef}).  The following formula is taken from \cite{prudnikov} 7.2.3(9)

\[
\pFq{p+1}{q+1}{\beta,,\alpha_{p}}{\beta+\sigma, ,\rho_{q}}{z}=\frac{\Gamma\left(\beta+\sigma\right)}{\Gamma\left(\beta\right)\Gamma\left(\sigma\right)}\int_{0}^{1}t^{\beta-1}(1-t)^{\sigma-1}\pFq{p}{q}{\alpha_{p}}{\rho_{q}}{zt}dt.
\]

Applying this to the $ {}_{4}F_{3}$ hypergeometric function from (\ref{F-d-coef}), with $z=1,$ $\beta = \frac{\alpha+m+1}{2}$ and $\sigma= \frac{m+d-\alpha}{2},$ (such that $\beta+\sigma = m+\frac{d+1}{2})$ we have that
\begin{equation}\label{hyp_int_rep}
\begin{aligned}
&\pFq{4}{3}{\frac{\alpha+m+1}{2},\frac{\alpha+m}{2},\frac{\tau+m}{2},\frac{\tau+m+1}{2}}{m+\frac{d+1}{2}, \frac{\alpha+\nu+\tau+m}{2},\frac{\alpha+\nu+\tau+m+1}{2}}{1}\\
&
=\frac{\Gamma\left(m+\frac{d+1}{2}\right)}{\Gamma\left(\frac{m+\alpha+1}{2} \right)\Gamma\left(\frac{m+d-\alpha}{2} \right)}
\int_{0}^{1} t^{\frac{m+\alpha-1}{2}}(1-t)^{\frac{m-\alpha+d-2}{2}}\pFq{3}{2}{\frac{\alpha}{2}+\frac{m}{2},\frac{\tau}{2}+\frac{m}{2},\frac{\tau+1}{2}+\frac{m}{2}}{\frac{\alpha+\nu+\tau}{2}+\frac{m}{2},\frac{\alpha+\nu+\tau+1}{2}+\frac{m}{2}}{t}dt
\end{aligned}
\end{equation}

The following identity is taken from \cite{luke1969special} 7.3(3)

 \begin{equation}
\label{integrand-help}
\begin{aligned}
&\pFq{p+1}{p}{a_{p+1}+r}{b_{p}+r}{t}=(1-t)^{\xi}\Bigl[1+\frac{d_{1}t}{2r}+\sum_{k=2}^{n}\frac{d_{k}}{r^{k}}+O\left(\frac{1}{r^{n+1}}\right)\Bigr],\\
&{\rm{where}}\quad \xi =\sum_{j=1}^{p}b_{j}-\sum_{j=1}^{p+1}a_{j}-r, \quad \quad \quad d_{1} =(\xi+r)^{2}+\sum_{j=1}^{p}b_{j}^{2}- \sum_{j=1}^{p+1}a_{j}^2, \\
& d_{k} = \sum_{s=1}^{k}\beta_{k,s}t^{s},\,\,\,\,\,(2\le k\le m),  \quad {\rm{and}} \quad\quad |\arg(1-t)|<\pi.
\end{aligned}
\end{equation}
 The quantities $\beta_{k,s}$ above depend only on the parameters of $a_{p+1}$ and $b_{p}.$ For $p=2$ we can use (\ref{integrand-help}) to write the $ {}_{3}F_{2}$ hypergeometric function appearing in the integral (\ref{hyp_int_rep}) as follows
 \[
 \pFq{3}{2}{\frac{\alpha}{2}+\frac{m}{2},\frac{\tau}{2}+\frac{m}{2},\frac{\tau+1}{2}+\frac{m}{2}}{\frac{\alpha+\nu+\tau}{2}+\frac{m}{2},\frac{\alpha+\nu+\tau+1}{2}+\frac{m}{2}}{t}
 =(1-t)^{\frac{\alpha-m}{2}+\nu}\Bigl[1+\frac{d_{1}t}{m}+\frac{\beta_{2,1}t+\beta_{2,2}t^{2}}{\left(\frac{m}{2}\right)^2}+O\left(\frac{1}{m^3}\right)\Bigr].
 \]

 We can use the above to write
 \begin{equation}\label{ExpansionF}
 \begin{aligned}
 &\pFq{4}{3}{\frac{\alpha+m+1}{2},\frac{\alpha+m}{2},\frac{\tau+m}{2},\frac{\tau+m+1}{2}}{m+\frac{d+1}{2}, \frac{\alpha+\nu+\tau+m}{2},\frac{\alpha+\nu+\tau+m+1}{2}}{1}\\
 &=\frac{\Gamma\left(m+\frac{d+1}{2}\right)}{\Gamma\left(\frac{m+\alpha+1}{2} \right)\Gamma\left(\frac{m+d-\alpha}{2} \right)}\Bigl[
 I_{m,d,\nu,\alpha}(0)+\left(\frac{d_{1}}{m}+\frac{4\beta_{1,2}}{m^2}\right) I_{m,d,\nu,\alpha}(1)+\frac{4\beta_{2,2}}{m^{2}}I_{m,d,\nu,\alpha}(2)+ O\left(\frac{1}{m^3}\right)\Bigr],
\end{aligned}
\end{equation}
where
\[
\begin{aligned}
 I_{m,d,\nu,\alpha}(j)&=\int_{0}^{1} t^{\frac{m+\alpha-1}{2}+j}(1-t)^{\frac{d-2}{2}+\nu}dt\\
 &=B\left(\frac{m}{2}+\frac{\alpha+1}{2}+j,\frac{d}{2}+\nu\right)=\frac{\Gamma\left(\frac{m}{2}+\frac{\alpha+1}{2}+j\right)\Gamma\left(\frac{d}{2}+\nu\right)}{\Gamma\left( \frac{m}{2}+\frac{\alpha+1+d}{2}+j+\nu\right)}, \quad j=0,1,2.
\end{aligned}
\]
In the case where $m$ is large  we can apply Stirling's asymptotic formula (\ref{Stirling}) to deduce that
\[
   I_{m,d,\nu,\alpha}(j)\sim \frac{\Gamma\left(\frac{d}{2}+\nu\right)2^{\frac{d}{2}+\nu}}{m^{\frac{d}{2}+\nu}},\quad j=0,1,2,\quad {\rm{and}}\quad  \frac{\Gamma\left(m+\frac{d+1}{2}\right)}{\Gamma\left(\frac{m+\alpha+1}{2} \right)\Gamma\left(\frac{m+d-\alpha}{2} \right)}\sim \frac{2^{\frac{d-1}{2}}}{\sqrt{2\pi}}2^{m}\sqrt{m}.
   \]
These two  results allow us to deduce that, when $n$ is large, we have the following asymptotic formula
\begin{equation}\label{hypass}
\pFq{4}{3}{\frac{\alpha+m+1}{2},\frac{\alpha+m}{2},\frac{\tau+m}{2},\frac{\tau+m+1}{2}}{m+\frac{d+1}{2}, \frac{\alpha+\nu+\tau+m}{2},\frac{\alpha+\nu+\tau+m+1}{2}}{1}\sim
\frac{\Gamma\left(\frac{d}{2}+\nu\right)}{2\sqrt{\pi}}\frac{2^{m+d+\nu}}{m^{\frac{d-1}{2}+\nu}}\left(1+\frac{d_{1}}{m}+\frac{(\beta_{2,1}t+\beta_{2,2}t^2)}{m^{2}} \right).
\end{equation}

Bringing (\ref{hypass}) and (\ref{Const_decay}) together, we can conclude that
\[
b_{m,d}\sim \frac{\Gamma(\nu+\alpha)\Gamma(\nu+\tau)}{\Gamma(\alpha)\Gamma(\nu)\Gamma(\tau)}
\frac{2^{\nu+1}\Gamma\left(\frac{d}{2}+\nu\right)}{\Gamma\left(\frac{d}{2}\right)}\frac{1}{m^{1+2\nu}}.
\]

Lemma \ref{RKHSlem} shows that the  native space  $N_{\psi_{{\cal F}}}$ possesses the stated  reproducing kernel properties.  The norm equivalence of    $N_{\psi_{{\cal F}}}$   to the Sobolev space of order $ \nu+\frac{d}{2}$ follows from Remark \ref{onSob} and the established asymptotic decay rate of $b_{m,d,{\cal F}}(\blambda).$
 \end{proof}

\addtocontents{toc}{\vskip 4mm}
\addcontentsline{toc}{section}{\protect\numberline{}References} 
\bibliographystyle{chicago}

\



\bibliography{Bibliography_cor}

\end{document}